\documentclass[a4, 10pt]{amsart}
\usepackage[foot]{amsaddr}

\usepackage{amssymb}
\usepackage{amstext}
\usepackage{amsmath}
\usepackage{amscd}
\usepackage{latexsym}
\usepackage{amsfonts}
\usepackage{color}
\usepackage{mathtools}
\usepackage{eqparbox}
\usepackage{enumerate}
\usepackage{textcomp}
\usepackage{url,hyperref}
\usepackage[all]{xy}
\usepackage{algorithm}
\usepackage{algpseudocode}
\usepackage{appendix}

\theoremstyle{plain}
\newtheorem{thm}{Theorem}[section]
\newtheorem*{thm*}{Theorem}
\newtheorem*{cor*}{Corollary}
\newtheorem*{defn*}{Definition}

\newtheorem{lem}[thm]{Lemma}
\newtheorem{cor}[thm]{Corollary}

\newtheorem*{claim*}{Claim}

\newtheorem*{note*}{Notation}

\theoremstyle{definition}
\newtheorem{defn}[thm]{Definition}
\newtheorem{ex}[thm]{Example}

\theoremstyle{remark}

\numberwithin{equation}{thm}

\def\mod{\mathrm{mod}}

\def\T{{\sf T}}

\newcommand{\calF}{\mathcal{F}}

\def\argmin{\mathrm{argmin}}

\newtheorem{manualtheoreminner}{\bf Theorem}
\newenvironment{manualtheorem}[1]{%
  \renewcommand\themanualtheoreminner{#1}%
  \manualtheoreminner
}{\endmanualtheoreminner}

\begin{document}

\setlength{\baselineskip}{17pt}

\title{On the Convergence Proof of AMSGrad and a New Version$^\star$}

\author{Tran Thi Phuong$^{(*,**,***)}$}
\address[$*$]{Faculty of Mathematics and Statistics, Ton Duc Thang University, Ho Chi Minh City, Vietnam. Postal address: 19 Nguyen Huu Tho street, Tan Phong ward, District 7, Ho Chi Minh City, Vietnam.}
\address[$**$]{Meiji University. Postal address: 1-1-1 Higashi-Mita, Tama-ku, Kawasaki-shi, Kanagawa 214-8571, Japan.}
\email{tranthiphuong@tdtu.edu.vn}

\author{Le Trieu Phong$^{(***)}$}
\address[$***$]{National Institute of Information and Communications Technology (NICT). Postal address: 4-2-1, Nukui-Kitamachi, Koganei, Tokyo 184-8795, Japan.}
\email{phong@nict.go.jp}
\thanks{$^\star$A version of this paper appears at IEEE Access DOI: \href{https://ieeexplore.ieee.org/document/8713445?source=authoralert}{10.1109/ACCESS.2019.2916341}}

\maketitle

\begin{abstract}The adaptive moment estimation \textcolor{black}{algorithm} Adam (Kingma and Ba) is a popular optimizer in the training of deep neural networks. However, Reddi et al. have recently shown that the convergence proof of Adam is problematic and proposed a variant of Adam called AMSGrad as a fix. In this paper, we show that  the convergence proof of AMSGrad is also problematic. {Concretely, the problem in the convergence proof of AMSGrad is in handling the hyper-parameters, treating them as equal while they are not. This is also the neglected issue in the convergence proof of Adam. We provide an explicit counter-example of a simple convex optimization setting to show this neglected issue. Depending on manipulating the hyper-parameters}, we present various fixes for this issue. {We provide a new convergence proof for AMSGrad as the first fix}. We also propose a new version of AMSGrad called AdamX as another fix. {Our experiments on the benchmark dataset also support our theoretical results.}

\medskip
\noindent {\sc \keywordsname.}  Optimizer, adaptive moment estimation, Adam, AMSGrad, deep neural networks.
\end{abstract}
\tableofcontents
\section{Introduction and our contributions}
One of the most popular algorithms for training deep neural networks is stochastic gradient descent (SGD) \cite{SGD} and its variants. Among the various variants of SGD, the algorithm with the adaptive moment  estimation Adam \cite{KingmaB14} is widely used  in practice. However, Reddi et al. \cite{AMSGrad} have recently shown that the convergence proof of Adam is problematic and proposed a variant of Adam called AMSGrad  to solve this issue. 

{\medskip \noindent \bf Our contribution.}
In this paper, we point out a flaw in the convergence proof of AMSGrad. We then fix this flaw by providing a new convergence  proof for AMSGrad in the case of special parameters. In addition, in the case of general parameters, we propose a new and slightly modified version of AMSGrad.


To provide more details, let us recall AMSGrad in Algorithm \ref{AMSGrad}, in which the mathematical notation can be fully  found in Section \ref{prem}.

 \begin{algorithm}[H]\label{alg2}
    \caption{AMSGrad (Reddi et al. \cite{AMSGrad}).}\label{AMSGrad}
    \begin{algorithmic}
    	\State\hspace{-\algorithmicindent} \textbf{Input:} $x_1\in \mathcal F$, step size $\{\alpha_t\}_{t=1}^T, \{\beta_{1,t}\}_{t=1}^T, \beta_2$
	\State \hspace{-\algorithmicindent} Set $m_0 = 0, v_0 = 0$, and $\hat v_0 = 0$
   	\For {$(t=1; t\le T; t\gets t+1)$}
		\State $g_t = \nabla f_t(x_{t})$ 
		\State $m_t = \beta_{1,t}\cdot m_{t-1} + (1-\beta_{1,t})\cdot g_t$ 
		\State $v_t = \beta_2\cdot v_{t-1} + (1-\beta_2)\cdot g^2_t$ 
		\State $\hat v_t = \max(\hat v_{t-1} , v_{t})$ and $\hat V_t = \text{diag}(\hat v_t)$
		\State $x_{t+1} =  \prod_{\mathcal F, \sqrt{\hat V_t}}(x_t - \alpha_t \cdot m_t/\sqrt{\hat v_t}) $ 
    	\EndFor
    \end{algorithmic}
    \end{algorithm}

The main theorem for the convergence of AMSGrad in \cite{AMSGrad} is as follows. To simplify the notation, we define $g_t \overset{\Delta}{=} \nabla_{x} f_t(x_t)$, $g_{t,i}$ as the $i^{\text{th}}$ element of $g_t$ and $g_{1:t,i} \in\mathbb R^t$ as a vector that contains the $i^{\text{th}}$ dimension of the gradients over all iterations up to $t$, namely, $g_{1:t,i} = [g_{1,i}, g_{2,i},...,g_{t,i}]$. 

\begin{manualtheorem}{A}[Theorem 4 in \cite{AMSGrad}, {\bf problematic}]\label{mainthmAMSGrad} {\it Let $x_t$ and $v_t$ be the sequences obtained from Algorithm \ref{AMSGrad}, $\alpha_t = \frac{\alpha}{\sqrt{t}}$, $\beta_1 = \beta_{1,1}$, $\beta_{1,t} \le \beta_1$ for all $t\in[T]$ and $\frac{\beta_1}{\sqrt{\beta_2}} \le 1$. Assume that $\mathcal F$ has bounded diameter $D_{\infty}$ and $\lVert{\nabla f_t(x)}\rVert_{\infty} \le G_{\infty}$ for all $t\in [T]$ and $x\in \mathcal F$. For $x_t$ generated using AMSGrad (Algorithm \ref{AMSGrad}), we have the following bound on the regret:
\begin{eqnarray*}
R(T) & \le&  \frac{D_{\infty}^2\sqrt{T}}{\alpha(1-\beta_{1})}\sum_{i=1}^{d} \sqrt{\hat{v}_{T,i}}  + \frac{D_{\infty}^2}{2(1-\beta_{1})} \sum_{i=1}^{d} \sum_{t=1}^{T}\frac{\beta_{1,t}\sqrt{\hat v_{t,i}}}{\alpha_{t}}+  \frac{\alpha\sqrt{ 1+\ln T}}{(1-\beta_1)^2(1-\gamma)\sqrt{1-\beta_2}} \sum_{i=1}^{d}\lVert{g_{1:T, i}}\rVert_2.
\end{eqnarray*}
}
\end{manualtheorem}

\medskip \noindent  In their proof for Theorem \ref{mainthmAMSGrad}, Reddi et al. resolved an issue on the so-called {\it telescopic sum} in the convergence proof of Adam (\cite[Theorem 10.5]{KingmaB14}). Specifically, Reddi et al. adjusted $\hat v_t$ such that \textcolor{black}{all components in the vector} 
\begin{eqnarray}
\frac{\sqrt{\hat v_{t+1}}}{\alpha_{t+1}} - \frac{\sqrt{\hat v_{t}}}{\alpha_{t}} \label{Reddi_trick}
\end{eqnarray}
are {\em always} positive. However, there is {\em another} issue \textcolor{black}{(showed in Section \ref{prob_AMSGrad})} in the convergence proof of Adam that AMSGrad unfortunately neglects. \textcolor{black}{The issue}  affects both the correctness of Reddi et al.'s proof and  the upper bound for the regret in Theorem \ref{mainthmAMSGrad}. To deal with the issue in a general way, we propose to modify Algorithm \ref{AMSGrad} such that all components in the vector
$$\frac{\sqrt{\hat{v}_{t+1}}}{\alpha_{t+1}(1-\boxed{\beta_{1, t+1}})} - \frac{\sqrt{\hat{v}_{t}}}{\alpha_{t}(1-\boxed{\beta_{1, t}})}$$
are {\em always} positive. The differences with (\ref{Reddi_trick}) are highlighted in the boxes for clarity. 

{\medskip \noindent \bf Paper roadmap.} We begin with preliminaries in Section \ref{prem}. We  show where the proof of Theorem \ref{mainthmAMSGrad} becomes invalid  in Section \ref{prob_AMSGrad}. After that, we  suggest two ways to resolve the issue  in Sections \ref{new_con_AMSGrad} and \ref{new_version}. 

{\medskip \noindent \bf Subsequent works.}
The first version of this paper publicly appeared on arXiv on 7 April 2019. On 19 April 2019, Reddi et al. revised their original proofs\footnote{https://arxiv.org/abs/1904.09237}. The revised proof does not suffer from the issue pointed out in Section \ref{prob_AMSGrad} of this paper, although yielding a constant factor missing in the original claims.

\section{Preliminaries}\label{prem}

{\medskip \noindent \bf Notation.} Given a sequence of vectors $\{x_t\}_{1\le t\le T} (1\le T\in \mathbb N)$ in $\mathbb R^d$, we denote its $i^{\text{th}}$ coordinate by $x_{t,i}$ and use $x_t^k$ to denote the elementwise power of $k$ and $\lVert{x_t}\lVert_2$, resp. $\lVert{x_t}\lVert_{\infty}$, to denote its $\ell_{2}$-norm, resp. $\ell_{\infty}$-norm. Let $\mathcal F \subseteq \mathbb R^d$ be a feasible set of points such that $\mathcal F$ has bounded diameter $D_{\infty}$, that is, $\lVert{x-y}\rVert_{\infty} \le D_{\infty}$ for all $x,y\in \mathcal F$, and $\mathcal S^d_+$ denote the set of all positive definite $d\times d$ matrices. {For a matrix $A\in \mathcal S^d_+$, we denote $A^{1/2}$ for the square root of $A$.} 
The projection operation $\prod_{\mathcal F, A} (y)$ for $A \in \mathcal S^d_+$ is defined as $\argmin_{ x\in \mathcal F}\lVert{A^{1/2} (x-y)}\lVert_2$ for all $y \in \mathbb R^d$. When $d=1$ and $\mathcal F \subset \mathbb R$, the positive definite matrix $A$ is a positive number, so that the projection $\prod_{\mathcal F, A} (y)$ becomes $\argmin_{x\in \mathcal F}|x-y|$.
We use $\langle x, y \rangle$ to denote the inner product between $x$ and $y\in \mathbb R^d$. The gradient of a function $f$ evaluated at $x\in \mathbb R^d$ is denoted by $\nabla f(x)$. {For vectors $x, y\in \mathbb R^d$, we use $\sqrt{x}$ or $x^{1/2}$ for element-wise square root, $x^2$ for element-wise square, $x/y$ to denote element-wise division. For an integer $n\in \mathbb N$, we denote by $[n]$ the set of integers $\{1,2,...,n\}$.}

{\medskip \noindent \bf Optimization setup.} Let $f_1, f_2,..., f_T: \mathbb R^d \to \mathbb R$ be an arbitrary sequence of convex cost functions and $x_1\in \mathbb R^d$. At each time $t\ge 1$, the goal is to predict the parameter $x_t$ and evaluate it on a previously unknown cost function $f_t$. Since the nature of the sequence is unknown in advance, the algorithm is evaluated by using the regret, that is, the sum of all the previous differences between the online prediction $f_t(x_t)$ and the best fixed-point parameter $f_t(x^{*})$ from a feasible set $\calF$ for all the previous steps. Concretely, the regret is defined as 
$$R(T) = \sum_{t=1}^T[f_t(x_t) -f_t(x^{*})],$$
where $x^{*} = \text{argmin}_{x\in \calF}\sum_{t=1}^Tf_t(x)$.

\begin{defn} A function $f: \mathbb R^d \rightarrow \mathbb R$ is convex if for all $x, y\in \mathbb R^d$, and all $\lambda \in [0,1]$,
$$\lambda f(x) + (1-\lambda)f(y) \ge f(\lambda x + (1-\lambda)y).$$
\end{defn}

\begin{lem}\label{convex} If a function $f: \mathbb R^d \rightarrow \mathbb R$ is convex, then for all $x, y\in \mathbb R^d$,
$$f(y) \ge f(x) + \nabla f(x)^{\T}(y-x),$$
where $(-)^{\T}$ denotes the transpose of $(-)$.
\end{lem}

\begin{lem}[\bf{Cauchy--Schwarz inequality}]\label{CS} For all $n\ge 1$,  $u_i, v_i\in \mathbb R (1\le i \le n)$,
$$\left(\sum_{i=1}^n u_iv_i\right)^2 \le \left(\sum_{i=1}^nu_i^2\right) \left(\sum_{i=1}^nv_i^2\right).$$
\end{lem}

\begin{lem}[\bf{Taylor series}]\label{Taylor} For $\alpha\in \mathbb R$ and $0<\alpha<1$,
$$\sum_{t\ge 1}{\alpha^{t}} = \frac{1}{1-\alpha}$$
and $$\sum_{t\ge 1}{t\alpha^{t-1}} = \frac{1}{(1-\alpha)^2}.$$
\end{lem}

\begin{lem}[\bf{Upper bound for the harmonic series}] \label{Harmonic} For $N\in \mathbb N$,
$$\sum_{n=1}^{N} \frac{1}{n}\le \ln N +1.$$
\end{lem}

\begin{lem}\label{sqrt} For $N\in \mathbb N$,
$$\sum_{n=1}^{N} \frac{1}{\sqrt n}\le 2\sqrt N.$$
\end{lem}

\begin{lem}\label{sum} For all $n\in \mathbb N$ and {$a_i, b_i \in \mathbb R$ such that $a_i\ge 0$ and $b_i>0$ for all $i\in[n]$},
$$\frac{\sum_{i=1}^na_i}{\sum_{j=1}^n b_j} \le \sum_{i=1}^{n}\frac{a_i}{b_i}.$$
\end{lem}

\begin{lem}\label{McM&Str}{\cite[Lemma 3 in arXiv version]{AMSGrad}} For any $Q \in \mathcal S^d_+$ and convex feasible set $\mathcal F\subseteq \mathbb R^d$, suppose $u_1 = \min_{x\in \mathcal F}\lVert Q ^{1/2}(x-z_1)\rVert$ and $u_2 = \min_{x\in \mathcal F}\lVert Q ^{1/2}(x-z_2)\rVert$. Then, we have $$\lVert Q ^{1/2}(u_1-u_2)\rVert \le \lVert Q ^{1/2}(z_1-z_2)\rVert.$$
\end{lem}

\section{Issue  in the convergence proof of AMSGrad}\label{prob_AMSGrad}
Before showing the issue  in the convergence proof of AMSGrad, let us recall and prove the following \textcolor{black}{inequality, which also appears in \cite{AMSGrad}}.

\begin{lem}\label{prepare_lem} Algorithm \ref{AMSGrad} achieves the following guarantee, for all $T\ge 1$:
\begin{eqnarray*} 
R(T)
& \le& \sum_{i=1}^{d} \sum_{t=1}^{T} \frac{\sqrt{\hat v_{t,i}}}{ 2\alpha_t(1-\beta_{1,t})  }\left( ( x_{t,i} - x^{*}_{,i} )^2 - (x_{t+1, i} - x^{*}_{,i})^2 \right)\\
& & +  \sum_{i=1}^{d} \sum_{t=1}^{T} \frac{\alpha_t}{1-\beta_{1}} \frac{ m_{t,i}^2}{\sqrt{\hat v_{t,i}}}   +\sum_{i=1}^{d} \sum_{t=2}^{T}\frac{\beta_{1,t}\sqrt{\hat{v}_{t-1,i}}}{2\alpha_{t-1}(1-\beta_{1})}  (x_{t,i} - x^{*}_{,i})^2.
\end{eqnarray*} 
\end{lem}

\begin{proof} We note that 
\begin{eqnarray*} 
x_{t+1} =  \prod_{\mathcal F, \sqrt{\hat V_t}}(x_t - \alpha_t \cdot {\hat V^{-1/2}_t}m_t) = \min_{x\in \mathcal F}\lVert \hat V^{1/4}(x-(x_t - \alpha_t \hat V^{-1/2}m_t))\rVert
\end{eqnarray*}
and $ \prod_{\mathcal F, \sqrt{\hat V_t}}(x^{*}) = x^{*}$ for all $x^{*} \in \mathcal F$. 
For all $1\le t \le T$, put $g_t = \nabla_{x} f_t(x_t)$. Using Lemma \ref{McM&Str} with $u_1 = x_{t+1}$ and $u_2 = x^{*}$, we have 
\begin{eqnarray*} 
\lVert \hat V^{1/4}(x_{t+1} - x^{*}) \rVert^2 
&\le  & \lVert \hat V^{1/4}(x_t - \alpha_t \hat V^{-1/2}m_t - x^{*}) \rVert^2 \\
& = & \lVert \hat V^{1/4}(x_{t} - x^{*}) \rVert^2  + \alpha^2_t \lVert \hat V^{-1/4}m_t\rVert^2 - 2\alpha_t\langle m_t, x_t- x^*\rangle\\
& = & \lVert \hat V^{1/4}(x_{t} - x^{*}) \rVert^2  + \alpha^2_t \lVert \hat V^{-1/4}m_t\rVert^2 - 2\alpha_t\langle \beta_{1,t}m_{t-1} + (1-\beta_{1,t})g_t, x_t- x^*\rangle.
\end{eqnarray*} 
This yields
\begin{eqnarray*} 
\langle g_t, x_t- x^*\rangle 
& \le & \frac{1}{2\alpha_t(1-\beta_{1,t})}\left[ \lVert \hat V^{1/4}(x_{t} - x^{*}) \rVert^2 -  \lVert \hat V^{1/4}(x_{t+1} - x^{*}) \rVert^2  \right] \\  
& &+\frac{\alpha_t}{2(1-\beta_{1,t})}\lVert \hat V^{-1/4}m_t\rVert^2 
 - \frac{\beta_{1,t}}{1-\beta_{1,t}}\langle m_{t-1}, x_t- x^*\rangle. 
\end{eqnarray*} 
Therefore, we obtain
\begin{eqnarray} \label{tvheq}
\sum_{i=1}^{d} g_{t,i} (x_{t,i} - x^{*}_{,i}) 
& \le & \sum_{i=1}^{d} \frac{\sqrt{\hat v_{t,i}}}{2\alpha_t(1-\beta_{1,t}) } \left(( x_{t,i} - x^{*}_{,i} )^2  - (x_{t+1, i} - x^{*}_{,i})^2\right)\\
& & + \sum_{i=1}^{d} \frac{\alpha_t}{2(1-\beta_{1,t}) } \frac{ m^2_{t,i}}{\sqrt{\hat v_{t,i}}} 
-  \sum_{i=1}^{d}\frac{\beta_{1,t}}{1 - \beta_{1,t}}m_{t-1,i}( x_{t,i} - x^{*}_{,i}).\nonumber
\end{eqnarray} 
Moreover, by Lemma \ref{convex}, we have $f_t(x^{*}) -f_t(x_t) \ge g_t^{\text T}(x^{*}-x_t)$, where $g_t^{\T}$ denotes the transpose of vector $g_t$. This means that
$$ f_t(x_t) - f_t(x^{*}) \le g_t^{\T}(x_t-x^{*}) = \sum_{i=1}^d g_{t,i}(x_{t,i}-x^{*}_{,i}).$$
Hence,
\begin{eqnarray} \label{regreteq}
R(T)  =  \sum_{t=1}^{T} [f_t(x_t) - f_t(x^{*})]
 \le \sum_{t=1}^{T}g_{t}^{\T} (x_{t} - x^{*})
 =  \sum_{t=1}^{T}\sum_{i=1}^{d}g_{t,i} (x_{t,i} - x^{*}_{,i}).
\end{eqnarray}
%
Combining (\ref{tvheq}) with (\ref{regreteq}), we obtain
\begin{eqnarray*} 
R(T) 
& \le&   \sum_{i=1}^{d} \sum_{t=1}^{T} \frac{\sqrt{\hat v_{t,i}}}{2\alpha_t(1-\beta_{1,t}) } \left(( x_{t,i} - x^{*}_{,i} )^2  - (x_{t+1, i} - x^{*}_{,i})^2\right)\\
& & +  \sum_{i=1}^{d} \sum_{t=1}^{T}  \frac{\alpha_t}{2(1-\beta_{1,t}) } \frac{ m^2_{t,i}}{\sqrt{\hat v_{t,i}}} +  \sum_{i=1}^{d} \sum_{t=2}^{T}  \frac{\beta_{1,t}}{1-\beta_{1,t}}m_{t-1,i}(x^{*}_{,i} - x_{t,i}).
\end{eqnarray*} 
On the other hand, for all $t\ge 2$, we have
\begin{eqnarray*} 
m_{t-1,t}( x^{*}_{,i} - x_{t,i})
& = & \frac{(\hat{v}_{t-1,i})^{1/4}}{\sqrt{\alpha_{t-1}}}  ( x^{*}_{,i} - x_{t,i})  \sqrt{\alpha_{t-1}} \frac{m_{t-1,i}}{(\hat{v}_{t-1,i})^{1/4}}  \\
& \le & \frac{\sqrt{\hat{v}_{t-1,i}}}{2\alpha_{t-1}}  (x_{t,i} - x^{*}_{,i})^2 + {\alpha_{t-1}} \frac{m^2_{t-1,i}}{2\sqrt{\hat{v}_{t-1,i}}},
\end{eqnarray*}
where the inequality is from the fact that $ab\le a^2/2 + b^2/2$ for any $a,b$. Hence, 
\begin{eqnarray*} 
\sum_{i=1}^{d} \sum_{t=1}^{T}  \frac{\beta_{1,t}}{1-\beta_{1,t}}m_{t-1,i}(x^{*}_{,i} - x_{t,i}) 
& = & \sum_{i=1}^{d} \sum_{t=2}^{T}  \frac{\beta_{1,t}}{1-\beta_{1,t}}m_{t-1,i}(x^{*}_{,i} - x_{t,i}) \\
& \le & \sum_{i=1}^{d} \sum_{t=2}^{T}\frac{\beta_{1,t}\sqrt{\hat{v}_{t-1,i}}}{2(1-\beta_{1,t})\alpha_{t-1}}  (x_{t,i} - x^{*}_{,i})^2  \\
& & + \sum_{i=1}^{d} \sum_{t=2}^{T} \frac{\beta_{1,t}{\alpha_{t-1}}}{2(1-\beta_{1,t})} \frac{m^2_{t-1,i}}{\sqrt{\hat{v}_{t-1,i}}}.
\end{eqnarray*}
Therefore, we obtain
\begin{eqnarray*} 
R(T)
& \le& \sum_{i=1}^{d} \sum_{t=1}^{T} \frac{\sqrt{\hat v_{t,i}}}{ 2\alpha_t(1-\beta_{1,t})  }\left( ( x_{t,i} - x^{*}_{,i} )^2 - (x_{t+1, i} - x^{*}_{,i})^2 \right)\\
& & +  \sum_{i=1}^{d} \sum_{t=1}^{T} \frac{\alpha_t}{2(1-\beta_{1,t})} \frac{ m_{t,i}^2}{\sqrt{\hat v_{t,i}}} + \sum_{i=1}^{d} \sum_{t=2}^{T}\frac{\beta_{1,t}\alpha_{t-1}}{2(1-\beta_{1,t})} \frac{m^2_{t-1,i}}{\sqrt{\hat{v}_{t-1,i}}}\\
& & +  \sum_{i=1}^{d} \sum_{t=2}^{T}\frac{\beta_{1,t}\sqrt{\hat{v}_{t-1,i}}}{2\alpha_{t-1}(1-\beta_{1,t})}  (x_{t,i} - x^{*}_{,i})^2.
\end{eqnarray*} 
Since $\beta_{1,t} \le \beta_1 (1\le t\le T)$, we obtain
\begin{eqnarray*}
\sum_{i=1}^{d} \sum_{t=2}^{T}\frac{\beta_{1,t}\sqrt{\hat{v}_{t-1,i}}}{2\alpha_{t-1}(1-\beta_{1,t})}  (x_{t,i} - x^{*}_{,i})^2 &\le & \sum_{i=1}^{d} \sum_{t=2}^{T}\frac{\beta_{1,t}\sqrt{\hat{v}_{t-1,i}}}{2\alpha_{t-1}(1-\beta_{1})}  (x_{t,i} - x^{*}_{,i})^2.
 \end{eqnarray*}
 Moreover, we have
\begin{eqnarray*}
\sum_{i=1}^{d} \sum_{t=2}^{T}\frac{\beta_{1,t}\alpha_{t-1}}{2(1-\beta_{1,t})} \frac{m^2_{t-1,i}}{\sqrt{\hat{v}_{t-1,i}}}
& = &  \sum_{i=1}^{d} \sum_{t=1}^{T-1}\frac{\beta_{1,t+1}\alpha_{t}}{2(1-\beta_{1,t+1})} \frac{m^2_{t,i}}{\sqrt{\hat{v}_{t,i}}}\\
& \le & \sum_{i=1}^{d} \sum_{t=1}^{T} \frac{\alpha_t}{2(1-\beta_{1,t+1})} \frac{ m_{t,i}^2}{\sqrt{\hat v_{t,i}}}\\
& \le & \sum_{i=1}^{d} \sum_{t=1}^{T} \frac{\alpha_t}{2(1-\beta_{1})} \frac{ m_{t,i}^2}{\sqrt{\hat v_{t,i}}},\\
\end{eqnarray*}
where the last inequality is from the assumption that $\beta_{1,t} \le \beta_1 <1 (1\le t\le T)$. Therefore,
\begin{eqnarray*}
\sum_{i=1}^{d} \sum_{t=1}^{T} \frac{\alpha_t}{2(1-\beta_{1,t})} \frac{ m_{t,i}^2}{\sqrt{\hat v_{t,i}}} + \sum_{i=1}^{d} \sum_{t=2}^{T}\frac{\beta_{1,t}\alpha_{t-1}}{2(1-\beta_{1,t})} \frac{m^2_{t-1,i}}{\sqrt{\hat{v}_{t-1,i}}} 
& \le & \sum_{i=1}^{d} \sum_{t=1}^{T} \frac{\alpha_t}{1-\beta_{1}} \frac{ m_{t,i}^2}{\sqrt{\hat v_{t,i}}}
\end{eqnarray*}
and we obtain the desired bound for $R(T)$.
\end{proof}

{\medskip \noindent \bf Issue in the convergence proof of  AMSGrad.} We denote the terms on the right hand-side of the upper bound for $R(T)$ in Lemma \ref{prepare_lem} as  
\begin{eqnarray}\label{eqmain}
\sum_{i=1}^{d} \sum_{t=1}^{T} \frac{\sqrt{\hat v_{t,i}}}{ 2\alpha_t(1-\beta_{1,t})  }\left( ( x_{t,i} - x^{*}_{,i} )^2 - (x_{t+1, i} - x^{*}_{,i})^2 \right),
\end{eqnarray}
\begin{eqnarray}\label{eqsecond}
 \sum_{i=1}^{d} \sum_{t=1}^{T} \frac{\alpha_t}{1-\beta_{1}} \frac{ m_{t,i}^2}{\sqrt{\hat v_{t,i}}},
\end{eqnarray}
and
\begin{eqnarray}\label{eqthird}
\sum_{i=1}^{d} \sum_{t=2}^{T}\frac{\beta_{1,t}\sqrt{\hat{v}_{t-1,i}}}{2\alpha_{t-1}(1-\beta_{1})}  (x_{t,i} - x^{*}_{,i})^2.
\end{eqnarray}

 The issue  in the proof of the convergence theorem of AMSGrad \cite[Theorem 4]{AMSGrad} becomes on examining the term (\ref{eqmain}). 
Indeed, in \cite[page 18]{AMSGrad}, Reddi et al. used\footnote{Concretely, on page 18 of \cite{AMSGrad}, it is stated that \lq\lq {\it The [...]  inequality   use the fact that $\beta_{1,t} \le \beta_{1}$.}"} the property that $\beta_{1,t} \le \beta_{1}$, and hence
$$\frac{1}{1-\beta_{1,t}} \leq \frac{1}{1-\beta_{1}},$$
 to replace all $\beta_{1,t}$ by $\beta_{1}$ as 
\begin{eqnarray*}
(\ref {eqmain}) & {\color{red}\le} & \sum_{i=1}^{d} \sum_{t=1}^{T} \frac{\sqrt{\hat v_{t,i}}}{ 2\alpha_t(1-{\color{black}\beta_{1}})  }\left( ( x_{t,i} - x^{*}_{,i} )^2 - (x_{t+1, i} - x^{*}_{,i})^2 \right) \\
 & \le &  \sum_{i=1}^{d} \frac{ \sqrt{\hat{v}_{1,i}}}{2\alpha_1(1-\beta_{1})}   (x_{1, i} - x^{*}_{,i})^2  + \frac{1}{2(1-\beta_{1})}\sum_{i=1}^{d} \sum_{t=2}^{T} (x_{t, i} - x^{*}_{,i})^2 \left( \frac{\sqrt{\hat{v}_{t,i}}}{\alpha_{t}} - \frac{\sqrt{\hat{v}_{t-1,i}}}{\alpha_{t-1}}  \right).
\end{eqnarray*}
However, the first inequality \textcolor{black}{(in red)}  is not guaranteed because  the quantity
$$( x_{t,i} - x^{*}_{,i} )^2 - (x_{t+1, i} - x^{*}_{,i})^2 $$ 
in (\ref{eqmain}) may be {\em both} negative {\em and} positive \textcolor{black}{as shown in Example \ref{exa_}}. This is also a neglected issue in the convergence proofs in Kingma and Ba \cite[Theorem 10.5]{KingmaB14}, Luo et al. \cite[Theorem 4]{LuoXiongLiu}, Bock et al. \cite[Theorem 4.4]{BoGoWe}, and Chen and Gu \cite[Theorem 4.2]{Padam}. 

\begin{ex}[for AMSGrad convergence  proof] \label{exa_}\normalfont We use the \textcolor{black}{function} in the Synthetic Experiment of Reddi et al. \cite[Page 6]{AMSGrad}
\begin{align*}
f_t(x)
            &=\begin{cases}
               \eqmakebox[W][l]{$1010x$,} & t ~\mod ~101 = 1\\
               \eqmakebox[W][l]{$-10x$,} & \text{otherwise, }
             \end{cases}
\end{align*}
with the constraint set $\mathcal F = [-1,1]$. The optimal solution is $x^*=-1$.
By the proof of \cite[Theorem 1]{AMSGrad}, the initial point $x_1 = 1$. By Algorithm \ref{AMSGrad}, $m_0 = 0$, $v_0 = 0$, and $\hat v_0 = 0$.  We  choose $\beta_1 = 0.9$, $\beta_{1,t} = \beta_1\lambda^{t-1}$, where $\lambda = 0.001$, $\beta_2 = 0.999$, and $\alpha_t = \alpha/\sqrt t$, where $\alpha = 0.001$. Under this setting, we have $f_1(x_1) = 1010x_1$, $f_2(x_2) = -10x_2$, $f_3(x_3) = -10x_3$ and hence 
\begin{align*}
g_1 &=  \nabla f_1(x_1) = 1010,\\
m_1 &=  \beta_{1,1}m_0 + (1-\beta_{1,1})g_1 = (1-0.9)1010 = 101,\\
v_1 & = \beta_2v_0 + (1-\beta_2)g_1^2 = (1-0.999)1010^2 = 1020.1,\\ 
\hat v_1 &= \max(\hat v_0, v_1) = v_1.
\end{align*}
Therefore,
\begin{align*}
x_1 - \alpha_1 m_1/\sqrt{\hat v_1}  &=  1-(0.001)101/\sqrt{1020.1}\\
& = 0.9968377223398316.
\end{align*}
Since $x_1 - \alpha_1 m_1/\sqrt{\hat v_1}>0$, we have
\begin{align*}
x_2 &= \prod_{\mathcal F}(x_1 - \alpha_1 m_1/\sqrt{\hat v_1}) \\
& = \min(1, x_1 - \alpha_1 m_1/\sqrt{\hat v_1}) \\
&= 0.9968377223398316.
\end{align*}
Hence, $$( x_1 - x^{*} )^2 - (x_2 - x^{*})^2  = 0.001264811064067839 >0.$$
At $t=2$, we have
\begin{align*}
g_2 &= -10,\\
m_2 &=  \beta_{1,2}m_1 + (1-\beta_{1,2})g_2 \\
& = (0.9)(0.001)(101) + [1-(0.9)(0.001)](-10) \\
&= -9.9001,\\
v_2 & =   \beta_2v_1 + (1-\beta_2)g_2^2\\
& = (0.999)(1020.1) + (1-0.999)(-10)^2 \\
&= 1019.1799000000001,\\
\hat v_2 &= \max(\hat v_1, v_2) = v_1\\
& = 1020.1.
\end{align*}
Therefore,
\begin{eqnarray*}
x_2 \!-\! \alpha_2 m_2/\sqrt{\hat v_2} &=&  0.9968377223398316 \!-\! \frac{0.001}{\sqrt 2} \frac{(-9.9001)}{\sqrt{1020.1}} \\
&=& 0.9970569034941291.
\end{eqnarray*}
Since $x_2 - \alpha_2 m_2/\sqrt{\hat v_2}>0$, we obtain
\begin{align*}
x_3& = \prod_{\mathcal F}(x_2 - \alpha_2 m_2/\sqrt{\hat v_2}) \\
&= \min(1, x_2 - \alpha_2 m_2/\sqrt{\hat v_2}) \\
&= 0.9970569034941291.
\end{align*}
Hence, $$( x_2 - x^{*} )^2 - (x_3 - x^{*})^2 = -0.0008753864342319062 <0.$$

\end{ex}

{\medskip \noindent \bf Outline of our solution.}
Let us  rewrite (\ref{eqmain}) as
\begin{eqnarray*}
(\ref {eqmain}) & = & \sum_{i=1}^{d} \frac{ \sqrt{\hat{v}_{1,i}}}{2\alpha_1(1-\beta_{1,1})}   (x_{1, i} - x^{*}_{,i})^2  +  \sum_{i=1}^{d} \sum_{t=2}^{T}\frac{ \sqrt{\hat{v}_{t,i}}}{2\alpha_t(1-\beta_{1,t})}   (x_{t, i} - x^{*}_{,i})^2 \\
 & & -  \sum_{i=1}^{d} \sum_{t=2}^{T}\frac{\sqrt{\hat{v}_{t-1,i}}}{2\alpha_{t-1}(1-\beta_{1,t-1})}   (x_{t, i} - x^{*}_{,i})^2  -  \sum_{i=1}^{d}\frac{\sqrt{\hat{v}_{T,i}}}{2\alpha_T(1-\beta_{1,T})}   (x_{T+1, i} - x^{*}_{,i})^2. \end{eqnarray*}
Omitting the term $\sum_{i=1}^{d}\frac{\sqrt{\hat{v}_{T,i}}}{2\alpha_T(1-\beta_{1,T})}   (x_{T+1, i} - x^{*}_{,i})^2 $, we obtain
\begin{eqnarray}\label{eqtemp2}
(\ref {eqmain})
& \le &  \sum_{i=1}^{d} \frac{ \sqrt{\hat{v}_{1,i}}}{2\alpha_1(1-\beta_{1,1})}   (x_{1, i} - x^{*}_{,i})^2  \\
&+& \frac{1}{2}\sum_{i=1}^{d} \sum_{t=2}^{T} (x_{t, i} - x^{*}_{,i})^2 \left( \frac{\sqrt{\hat{v}_{t,i}}}{\alpha_{t}(1-\boxed{\beta_{1,t}})} - \frac{\sqrt{\hat{v}_{t-1,i}}}{\alpha_{t-1}(1-\boxed{\beta_{1,t-1}})}  \right), \nonumber
\end{eqnarray}
in which the differences with Reddi et al. \cite{AMSGrad} are highlighted in the boxes, namely, $\boxed{\beta_{1,t}}$ and $\boxed{\beta_{1,t-1}}$ instead of $\beta_1$. 

 We suggest two ways to overcome these differences depending on the setting of $\beta_{1,t} (1\le t \le T)$:
\begin{itemize}
\item {\bf In Section \ref{new_con_AMSGrad}: } If either $\beta_{1,t} \overset{\Delta}{=} \beta_1\lambda^{t-1}$ or $\beta_{1,t} \overset{\Delta}{=}1/t$, $(1\le t \le T)$, where $0\le\beta_1< 1$ and $0<\lambda < 1$, then we give a new convergence theorem for AMSGrad in Section \ref{new_con_AMSGrad}. 
\item {\bf  In Section \ref{new_version}:} If the setting for  $\beta_{1,t} (1\le t \le T)$ is general, as in the statement of Theorem \ref{mainthmAMSGrad}, then we suggest a new (slightly modified) version for AMSGrad in Section \ref{new_version}.
\end{itemize}
\section{New convergence theorem for AMSGrad}\label{new_con_AMSGrad}
When either $\beta_{1,t} \overset{\Delta}{=} \beta_1\lambda^{t-1}$ or $\beta_{1,t} \overset{\Delta}{=}1/t$, $(1\le t \le T)$, where $0\le\beta_1< 1$ and $0<\lambda < 1$, Theorem \ref{mainthmAMSGrad} can be fixed as follows, in which the upper bounds of the regret $R(T)$ are changed.

\begin{thm}[Fixes for Theorem \ref{mainthmAMSGrad}] \label{mainthm_change_beta1} Let $x_t$ and $v_t$ be the sequences obtained from Algorithm \ref{AMSGrad}, $\alpha_t = \frac{\alpha}{\sqrt{t}}$, 
either $\beta_{1,t} = \beta_1\lambda^{t-1}$, where $\lambda \in (0,1)$, or $\beta_{1,t} = \frac{\beta_1}{t}$ for all $t\in[T]$ and $\gamma = \frac{\beta_1}{\sqrt{\beta_2}} \le 1$. Assume that $\mathcal F$ has bounded diameter $D_{\infty}$ and $\lVert{\nabla f_t(x)}\rVert_{\infty} \le G_{\infty}$ for all $t\in [T]$ and $x\in \mathcal F$. For $x_t$ generated using AMSGrad (Algorithm \ref{AMSGrad}), we have the following bound on the regret.
Then, there is some $1\le t_0 \le T$ such that AMSGrad achieves the following guarantee for all $T\ge 1$:
\begin{eqnarray*}
R(T) & \le&  \frac{dD_{\infty}^2G_{\infty}}{2\alpha(1-\beta_1)}\left( \sum_{t=1}^{t_0} \sqrt{t} + \sqrt T\right)+  \frac{d D_{\infty}^2G_{\infty}}{2\alpha(1-\beta_{1})(1-\lambda)^2} + \frac{\alpha\sqrt{ \ln T +1}}{(1-\beta_1)^2\sqrt{1-\beta_2}(1-\gamma)} \sum_{i=1}^{d}\lVert{g_{1:T, i}}\rVert_2,
\end{eqnarray*}
provided $\beta_{1,t} = \beta_1\lambda^{t-1}$, and
\begin{eqnarray*}
R(T) & \le& \frac{dD_{\infty}^2G_{\infty}}{2\alpha(1-\beta_1)}\left( \sum_{t=1}^{t_0} \sqrt{t} + \sqrt T\right)+  \frac{d D_{\infty}^2G_{\infty}\sqrt T}{\alpha(1-\beta_{1})} + \frac{\alpha\sqrt{ \ln T +1}}{(1-\beta_1)^2\sqrt{1-\beta_2}(1-\gamma)} \sum_{i=1}^{d}\lVert{g_{1:T, i}}\rVert_2 ,
\end{eqnarray*}
provided $\beta_{1,t} = \frac{\beta_1}{t}$.
\end{thm}


To prove Theorem \ref{mainthm_change_beta1}, we need the following Lemmas \ref{vt}, \ref{t_0}, and \ref{mainlem}. 

\begin{lem}\label{vt} $\sqrt{\hat{v}_{t}} \le G_{\infty}$.
\end{lem}
\begin{proof} From the definition of $\hat{v}_{t}$ in AMSGrad's algorithm, it is implied  that $\hat{v}_{t} = \max\{v_1,...,v_t\}$. Therefore, there is some $1\le s\le t$ such that $\hat{v}_{t} = v_s$. Hence,
\begin{eqnarray*} 
\sqrt{\hat{v}_{t}} & = & \sqrt{{v}_{s}}\\
& = & \sqrt{1-\beta_2}\sqrt{\sum_{k=1}^{s}\beta_2^{s-k}g^2_{k}}\\
&\le&  \sqrt{1-\beta_2}\sqrt{\sum_{k=1}^{s}\beta_2^{s-k} (\max_{1\le k\le s}{|g_{k}|})^2}\\
&=& G_{\infty}\sqrt{1-\beta_2} \sqrt{\sum_{k=1}^{s}\beta_2^{s-k}}\\
& \le & G_{\infty}\sqrt{1-\beta_2} \frac{1}{\sqrt{1-\beta_2}}\\
& = & G_{\infty},
\end{eqnarray*}
where the last inequality is by Lemma \ref{Taylor}.
\end{proof}

\begin{lem}\label{t_0} If either $\beta_{1,t} = \beta_1\lambda^{t-1}$ or $\beta_{1,t} = \beta_1/t$, then there exists some $t_0$ such that for every $t > t_0$,
$$ \frac{\sqrt{t \hat{v}_{t,i}}}{1-\beta_{1,t}} \ge \frac{\sqrt{(t-1)\hat{v}_{t-1,i}}}{1-\beta_{1,t-1}}.$$
\end{lem}
\begin{proof}
Since $\hat{v}_{t,i}\ge \hat{v}_{t-1,i}$, it is sufficient  to prove that there exists some $t_0$ such that for every $t > t_0$,
\begin{eqnarray*}
 \frac{\sqrt{t}}{1-\beta_{1,t}} &\ge &\frac{\sqrt{t-1}}{1-\beta_{1,t-1}}.
\end{eqnarray*}
In other word, 
\begin{eqnarray}\label{eqclaim}
 1- \frac{\beta_{1,t-1}-\beta_{1,t}}{1-\beta_{1,t}} &\ge &\sqrt{1-\frac{1}{t}}.
\end{eqnarray}

When $\beta_{1,t} = \beta_1/t$, from (\ref{eqclaim}) we have
\begin{eqnarray} \label{eqclaim2}
 1- \frac{\beta_1}{(t-1)(t-\beta_1)} &\ge &\sqrt{1-\frac{1}{t}}.
\end{eqnarray}

When $\beta_{1,t} = \beta_1\lambda^{t-1}$, (\ref{eqclaim}) have the following form
\begin{eqnarray} \label{eqclaim3}
  1- \frac{(1- \lambda) \beta_1\lambda^{t-2}}{1-\beta_1\lambda^{t-1}} = \frac{1-\beta_1 \lambda^{t-2}}{1-\beta_1 \lambda^{t-1}}& \ge& \sqrt{1 - \frac{1}{t}}.
\end{eqnarray}
Since $\beta_1$ and $\lambda$ are smaller than $1$,  it is easy to see that when $t$ is sufficiently large, meaning that $t > t_0$ for some $t_0$,  the left-hand side of (\ref{eqclaim2}) is $1 - O(1/t^2)$ and the left-hand side of (\ref{eqclaim3}) is larger than $1 - \beta_1 \lambda^{t-2} = 1 - O(\lambda^{t-2})$. Therefore, (\ref{eqclaim2}) and (\ref{eqclaim3}) hold  when $t$ is sufficiently large.
\end{proof}
\begin{lem} \label{mainlem} For the parameter settings and conditions assumed in Theorem \ref{mainthm_change_beta1}, we have
$$\sum_{t=1}^T\frac{{m}^2_{t,i}}{\sqrt{t\hat{v}_{t,i}}} \le \frac{\sqrt{ \ln T +1} }{(1-\beta_1)\sqrt{1-\beta_2}(1-\gamma)}\lVert{g_{1:T, i}}\rVert_2 .$$
\end{lem}

\begin{proof} The proof is almost identical to that of \cite[Lemma 2]{AMSGrad}. Since for all $ t\ge 1$, $\hat{v}_{t,i} \ge v_{t,i}$, we have
\begin{eqnarray*}
\frac{{m}^2_{t,i}}{\sqrt{t\hat{v}_{t,i}}} & \le& \frac{{m}^2_{t,i}}{\sqrt{t {v}_{t,i}}} \\
& = & \frac{[\sum_{k=1}^{t}(1-\beta_{1,k})(\prod_{j=k+1}^{t}\beta_{1,j})g_{k,i}]^2}{\sqrt{(1-\beta_2)t\sum_{k=1}^t\beta_2^{t-k}g^2_{k,i}}}\\
& \le & \frac{\left(\sum_{k=1}^{t}(1-\beta_{1,k})^2(\prod_{j=k+1}^{t}\beta_{1,j})\right)\left(\sum_{k=1}^{t}(\prod_{j=k+1}^{t}\beta_{1,j})g_{k,i}^2\right)}{\sqrt{(1-\beta_2)t\sum_{k=1}^t\beta_2^{t-k}g^2_{k,i}}}\\
& \le & \frac{\left(\sum_{k=1}^{t}\beta_{1}^{t-k}\right)\left(\sum_{k=1}^{t}\beta_{1}^{t-k}g_{k,i}^2\right)}{\sqrt{(1-\beta_2)t\sum_{k=1}^t\beta_2^{t-k}g^2_{k,i}}}\\
& \le & \frac{1}{(1-\beta_1)\sqrt{1-\beta_2}} \frac{\sum_{k=1}^t\beta_{1}^{t-k}g_{k,i}^2}{\sqrt{t\sum_{k=1}^t\beta_2^{t-k}g^2_{k,i}}},
\end{eqnarray*}
where the second inequality is by Lemma \ref{CS}, the third inequality is from the properties of $\beta_{1,k} \le 1$ and $\beta_{1,k} \le \beta_1$ for all $1\le k\le T$, and the fourth inequality is obtained  by applying Lemma \ref{Taylor} to $\sum_{k=1}^{t}\beta_{1}^{t-k}$.
Therefore,

\begin{eqnarray*}
\frac{{m}^2_{t,i}}{\sqrt{t\hat{v}_{t,i}}} 
& \le & \frac{1}{(1-\beta_{1})\sqrt{1-\beta_2}\sqrt t} \frac{\sum_{k=1}^t\beta_{1}^{t-k}g_{k,i}^2}{\sqrt{\sum_{k=1}^t\beta_2^{t-k}g^2_{k,i}}}\\
& \le & \frac{1}{(1-\beta_{1})\sqrt{1-\beta_2}\sqrt t}\sum_{k=1}^t\frac{\beta_{1}^{t-k}g_{k,i}^2}{\sqrt{\beta_2^{t-k}g^2_{k,i}}}\\
& \le & \frac{1}{(1-\beta_{1}) \sqrt{1-\beta_2}\sqrt t}\sum_{k=1}^t\frac{\beta_{1}^{t-k}}{\sqrt{\beta_2^{t-k}}} |{g_{k,i}}|\\
& = & \frac{1}{(1-\beta_{1}) \sqrt{1-\beta_2}\sqrt t}\sum_{k=1}^t\gamma^{t-k} |{g_{k,i}}|,
\end{eqnarray*}
where the second inequality is by Lemma \ref{sum}.
Therefore 
\begin{eqnarray} \label{eq4}
\sum_{t=1}^T\frac{{m}^2_{t,i}}{\sqrt{t\hat{v}_{t,i}}} 
&\le& \frac{1}{(1-\beta_1)\sqrt{1-\beta_2}} \sum_{t=1}^T\frac{1}{\sqrt{t}}\sum_{k=1}^t{\gamma}^{t-k} |{g_{k,i}}|.
\end{eqnarray}
It is sufficient  to consider $\sum_{t=1}^T \frac{1}{\sqrt{t}} \sum_{k=1}^t \gamma^{t-k}|{g_{k, i}}|$. Firstly, $\sum_{t=1}^T \frac{1}{\sqrt{t}} \sum_{k=1}^t \gamma^{t-k}|{g_{k, i}}|$ can be expanded as

\begin{eqnarray*}
\sum_{t=1}^T \frac{1}{\sqrt{t}} \sum_{k=1}^t\gamma^{t-k}|g_{k, i}| & = & \gamma^0|g_{1, i}| \\
& & + \frac{1}{\sqrt 2} \left[\gamma^1|g_{1, i}| + \gamma^0|{g_{2, i}}| \right]\\
& & + \frac{1}{\sqrt 3} \left[\gamma^2|g_{1, i}| + \gamma^1|{g_{2, i}}| + \gamma^0|{g_{3, i}}|\right]\\
& & +\cdots \\
& & + \frac{1}{\sqrt T} \left[\gamma^{T-1}|g_{1, i}| + \gamma^{T-2}|{g_{2, i}}| +...+ \gamma^0|{g_{T, i}}|\right].
\end{eqnarray*}

Changing the role of $|g_{1,i}|$ as the common factor, we obtain

\begin{eqnarray*}
\sum_{t=1}^T \frac{1}{\sqrt{t}} \sum_{k=1}^t \gamma^{t-k}|g_{k, i}| 
& = & |g_{1, i}| (\gamma^0 + \frac{1}{\sqrt 2}\gamma^1 + \frac{1}{\sqrt 3}\gamma^2 + ... + \frac{1}{\sqrt T}\gamma^{T-1}) \\
& & + |{g_{2, i}}| (\frac{1}{\sqrt 2}\gamma^0 + \frac{1}{\sqrt 3}\gamma^1 + ... + \frac{1}{\sqrt {T}}\gamma^{T-2})\\
& & + |{g_{3, i}}| (\frac{1}{\sqrt 3}\gamma^0 + \frac{1}{\sqrt 4}\gamma^1 + ... + \frac{1}{\sqrt {T}}\gamma^{T-3})\\
& & + \cdots\\
& & + |{g_{T, i}}| \frac{1}{\sqrt T}\gamma^0.
\end{eqnarray*}
In other words, 
$$\sum_{t=1}^T \frac{1}{\sqrt{t}} \sum_{k=1}^t \gamma^{t-k}|g_{k, i}|  = \sum_{t=1}^T |{g_{t, i}}| \sum_{k=t}^{T}\frac{1}{\sqrt k}\gamma^{k-t} $$
Moreover, since 
$$ \sum_{k=t}^{T}\frac{1}{\sqrt k}\gamma^{k-t} \le \sum_{k=t}^{T}\frac{1}{\sqrt t}\gamma^{k-t} = \frac{1}{\sqrt t}\sum_{k=t}^{T}\gamma^{k-t} = \frac{1}{\sqrt t}\sum_{k=0}^{T-t}\gamma^{k} \le \frac{1}{\sqrt t}\left( \frac{1}{1-\gamma}\right),$$
where the last inequality is by Lemma \ref{Taylor}, we obtain 
$$\sum_{t=1}^T \frac{1}{\sqrt{t}} \sum_{k=1}^t \gamma^{t-k}|g_{k, i}| \le  \sum_{t=1}^T|g_{t, i}| \frac{1}{\sqrt t}\left( \frac{1}{1-\gamma}\right) =  \frac{1}{1-\gamma} \sum_{t=1}^T \frac{1}{\sqrt t} |{g_{t, i}}| .$$
Furthermore, since 
$$ \sum_{t=1}^T \frac{1}{\sqrt t} |{g_{t, i}}| = \sqrt{\left(\sum_{t=1}^T \frac{1}{\sqrt t} |{g_{t, i}}|\right)^2} \le \sqrt{\sum_{t=1}^T \frac{1}{t}} \sqrt{\sum_{t=1}^Tg_{t, i}^2} \le (\sqrt{\ln T +1})\lVert{g_{1:T, i}}\rVert_2,$$
where the first inequality is by Lemma \ref{CS} and the last inequality is by Lemma \ref{Harmonic}, we obtain 
$$\sum_{t=1}^T \frac{1}{\sqrt{t}} \sum_{k=1}^t \gamma^{t-k}|g_{k, i}| \le \frac{\sqrt{\ln T +1}}{1-\gamma} \lVert{g_{1:T, i}}\rVert_2 . $$
Hence, by (\ref{eq4}),
$$\sum_{t=1}^T\frac{{m}^2_{t,i}}{\sqrt{t\hat{v}_{t,i}}} \le \frac{\sqrt{ \ln T + 1} }{(1-\beta_1)\sqrt{1-\beta_2}(1-\gamma)}\lVert{g_{1:T, i}}\rVert_2, $$
which ends the proof.
\end{proof}
Let us now prove Theorem \ref{mainthm_change_beta1}. 
\begin{proof}[{\bf Proof of Theorem \ref{mainthm_change_beta1}}] To prove Theorem \ref{mainthm_change_beta1}, by Lemma \ref{prepare_lem}, we need to bound the terms (\ref{eqmain}), (\ref{eqsecond}), and (\ref{eqthird}). First, we consider (\ref{eqsecond}). We have
\begin{eqnarray}\label{eqsecond2}
\sum_{i=1}^{d} \sum_{t=1}^{T} \frac{\alpha_t}{1-\beta_{1}} \frac{ m_{t,i}^2}{\sqrt{\hat v_{t,i}}} 
& = & \frac{\alpha}{1-\beta_{1}}\sum_{i=1}^{d} \sum_{t=1}^{T} \frac{ m_{t,i}^2}{\sqrt{t\hat v_{t,i}}}\\
& \le &\frac{\alpha\sqrt{ \ln T +1}}{(1-\beta_1)^2\sqrt{1-\beta_2}(1-\gamma)} \sum_{i=1}^{d}\lVert{g_{1:T, i}}\rVert_2\nonumber,
\end{eqnarray}
where the equality is by the assumption that $\alpha_t = \alpha/\sqrt t$ and the last inequality is by Lemma \ref{mainlem}. Next, we consider (\ref{eqthird}). The bound for (\ref{eqthird}) depends on either $\beta_{1,t} = \beta_1\lambda^{t-1} (0<\lambda<1 )$ or $\beta_{1,t} = \frac{\beta_1}{t}$. Recall that by assumption, $\lVert{x_m-x_n}\rVert_{\infty} \le D_{\infty}$ for any $m,n\in \{1,...,T\}$, $\alpha_t = \alpha/\sqrt{t}$.
If $\beta_{1,t} = \beta_1\lambda^{t-1} (0<\lambda<1 )$, then,
\begin{eqnarray}\label{beta_1lambda^{t-1}}
\sum_{i=1}^{d} \sum_{t=2}^{T}\frac{\beta_{1,t}\sqrt{\hat{v}_{t-1,i}}}{2\alpha_{t-1}(1-\beta_{1})}  (x_{t,i} - x^{*}_{,i})^2
& = & \sum_{i=1}^{d} \sum_{t=2}^{T}\frac{\beta_1\lambda^{t-1}\sqrt{(t-1)}\sqrt{\hat{v}_{t-1,i}} }{2\alpha(1-\beta_{1})}(x_{t,i} - x^{*}_{,i})^2\\
& \le & \frac{D_{\infty}^2G_{\infty}}{2\alpha(1-\beta_{1})} \sum_{i=1}^{d} \sum_{t=2}^{T}\sqrt{(t-1)} \lambda^{t-1}\nonumber\\
& \le & \frac{D_{\infty}^2G_{\infty}}{2\alpha(1-\beta_{1})} \sum_{i=1}^{d} \sum_{t=2}^{T}t \lambda^{t-1}\nonumber\\
& \le & \frac{D_{\infty}^2G_{\infty}}{2\alpha(1-\beta_{1})} \sum_{i=1}^{d} \frac{1}{(1-\lambda)^2}\nonumber\\
& = & \frac{d D_{\infty}^2G_{\infty}}{2\alpha(1-\beta_{1})(1-\lambda)^2}\nonumber,
\end{eqnarray}
where the first inequality is from Lemma \ref{vt} and the assumption that $\beta_1\le 1$, the last inequality is by Lemma \ref{Taylor}. 
If $\beta_{1,t} = \frac{\beta_1}{t}$, then, 
\begin{eqnarray}\label{frac{beta_1}{t}}
\sum_{i=1}^{d} \sum_{t=2}^{T}\frac{\beta_{1,t}\sqrt{\hat{v}_{t-1,i}}}{2\alpha_{t-1}(1-\beta_{1})}  (x_{t,i} - x^{*}_{,i})^2
& = & \sum_{i=1}^{d} \sum_{t=2}^{T}\frac{\beta_1\sqrt{(t-1)}\sqrt{\hat{v}_{t-1,i}} }{2\alpha(1-\beta_{1})t}(x_{t,i} - x^{*}_{,i})^2\\
& \le & \frac{D_{\infty}^2G_{\infty}}{2\alpha(1-\beta_{1})} \sum_{i=1}^{d} \sum_{t=2}^{T}\frac{\sqrt{(t-1)}}{t}\nonumber\\
& \le & \frac{D_{\infty}^2G_{\infty}}{2\alpha(1-\beta_{1})} \sum_{i=1}^{d} \sum_{t=2}^{T}\frac{1}{\sqrt{t}}\nonumber \\
& = & \frac{d D_{\infty}^2G_{\infty}\sqrt T}{\alpha(1-\beta_{1})}\nonumber,
\end{eqnarray}
where the first inequality is from Lemma \ref{vt} and the assumption that $\beta_1\le 1$, and the last inequality is by Lemma \ref{sqrt}. 

Finally, we will bound (\ref{eqmain}). From the inequality (\ref{eqtemp2}) and replacing $\alpha_t $ with $\frac{\alpha}{\sqrt t} (1\le t\le T)$, we obtain 
\begin{eqnarray*}
 (\ref {eqmain}) 
 & \le &  \sum_{i=1}^{d} \frac{ \sqrt{\hat{v}_{1,i}}}{2\alpha(1-\beta_{1})}   (x_{1, i} - x^{*}_{,i})^2  + 
 \frac{1}{2\alpha}\sum_{i=1}^{d} \sum_{t=2}^{T} (x_{t, i} - x^{*}_{,i})^2 \left( \frac{\sqrt{t \hat{v}_{t,i}}}{1-\beta_{1,t}} - \frac{\sqrt{(t-1)\hat{v}_{t-1,i}}}{1-\beta_{1,t-1}}  \right).
\end{eqnarray*}
By Lemma \ref{t_0}, there is some $t_0 (1\le t_0 \le T)$ such that  $\frac{\sqrt{t \hat{v}_{t,i}}}{1-\beta_{1,t}} \ge \frac{\sqrt{(t-1)\hat{v}_{t-1,i}}}{1-\beta_{1,t-1}}$ for all $t>t_0$. Therefore,
\begin{eqnarray*}
 (\ref {eqmain}) 
 & \le &  
 \sum_{i=1}^{d} \frac{ \sqrt{\hat{v}_{1,i}}}{2\alpha_1(1-\beta_{1,1})}   (x_{1, i} - x^{*}_{,i})^2 \\
 & & + \frac{1}{2\alpha}\sum_{i=1}^{d} \sum_{t=2}^{t_0} (x_{t, i} - x^{*}_{,i})^2 \left( \frac{\sqrt{t \hat{v}_{t,i}}}{1-\beta_{1,t}} - \frac{\sqrt{(t-1)\hat{v}_{t-1,i}}}{1-\beta_{1,t-1}}  \right)\\
 &  & + \frac{1}{2\alpha}\sum_{i=1}^{d} \sum_{t=t_0+1}^{T} (x_{t, i} - x^{*}_{,i})^2 \left( \frac{\sqrt{t \hat{v}_{t,i}}}{1-\beta_{1,t}} - \frac{\sqrt{(t-1)\hat{v}_{t-1,i}}}{1-\beta_{1,t-1}}  \right)\\
 & \le &  
 \frac{D_{\infty}^2}{2\alpha}\sum_{i=1}^{d} \frac{ \sqrt{\hat{v}_{1,i}}}{1-\beta_{1,1}} \\
 & & + \frac{1}{2\alpha}\sum_{i=1}^{d} \sum_{t=2}^{t_0} (x_{t, i} - x^{*}_{,i})^2 \left( \frac{\sqrt{t \hat{v}_{t,i}}}{1-\beta_{1,t}} - \frac{\sqrt{(t-1)\hat{v}_{t-1,i}}}{1-\beta_{1,t-1}}  \right)\\
 &  & + \frac{D_{\infty}^2}{2\alpha}\sum_{i=1}^{d} \sum_{t=t_0+1}^{T}  \left( \frac{\sqrt{t \hat{v}_{t,i}}}{1-\beta_{1,t}} - \frac{\sqrt{(t-1)\hat{v}_{t-1,i}}}{1-\beta_{1,t-1}}  \right).
 \end{eqnarray*}
 Since 
 \begin{eqnarray*}
 \frac{D_{\infty}^2}{2\alpha}\sum_{i=1}^{d} \sum_{t=t_0+1}^{T}  \left( \frac{\sqrt{t \hat{v}_{t,i}}}{1-\beta_{1,t}} - \frac{\sqrt{(t-1)\hat{v}_{t-1,i}}}{1-\beta_{1,t-1}}  \right)
 &=& 
 \frac{D_{\infty}^2}{2\alpha}\sum_{i=1}^{d} \frac{ \sqrt{T\hat{v}_{T,i}}}{1-\beta_{1, T}}  - \frac{D_{\infty}^2}{2\alpha}\sum_{i=1}^{d} \frac{ \sqrt{t_0\hat{v}_{t_0,i}}}{1-\beta_{1, t_0}}\\
 &\le& \frac{D_{\infty}^2}{2\alpha}\sum_{i=1}^{d} \frac{ \sqrt{T\hat{v}_{T,i}}}{1-\beta_{1, T}},
 \end{eqnarray*}
 we have
 \begin{eqnarray}\label{eqmain2}
  (\ref{eqmain})& \le & 
  \frac{D_{\infty}^2}{2\alpha}\sum_{i=1}^{d} \frac{ \sqrt{\hat{v}_{1,i}}}{1-\beta_{1,1}} + \frac{D_{\infty}^2}{2\alpha}\sum_{i=1}^{d}\frac{ \sqrt{T\hat{v}_{T,i}}}{1-\beta_{1, T}} \\
 & & + \frac{1}{2\alpha}\sum_{i=1}^{d} \sum_{t=2}^{t_0} (x_{t, i} - x^{*}_{,i})^2 \left( \frac{\sqrt{t \hat{v}_{t,i}}}{1-\beta_{1,t}} - \frac{\sqrt{(t-1)\hat{v}_{t-1,i}}}{1-\beta_{1,t-1}}  \right)\nonumber\\
 & \le &  
\frac{D_{\infty}^2}{2\alpha}\sum_{i=1}^{d} \frac{ \sqrt{\hat{v}_{1,i}}}{1-\beta_{1,1}} + \frac{D_{\infty}^2}{2\alpha}\sum_{i=1}^{d}\frac{ \sqrt{T\hat{v}_{T,i}}}{1-\beta_{1, T}}
 + \frac{D_{\infty}^2}{2\alpha}\sum_{i=1}^{d} \sum_{t=2}^{t_0} \frac{\sqrt{t \hat{v}_{t,i}}}{1-\beta_{1,t}}\nonumber\\
 & \le &\frac{dD_{\infty}^2G_{\infty}}{2\alpha(1-\beta_1)}\left(  \sum_{t=1}^{t_0} \sqrt{t}+ \sqrt T\right)\nonumber,
\end{eqnarray}
where the second  inequality is obtained by omitting the term $\frac{1}{2\alpha}\sum_{i=1}^{d} \sum_{t=2}^{t_0} (x_{t, i} - x^{*}_{,i})^2  \frac{\sqrt{(t-1)\hat{v}_{t-1,i}}}{1-\beta_{1,t-1}}$,
and the last inequality is by Lemma \ref{vt} and the assumption that $\beta_{1,t}\le  \beta_1 (1\le t\le T)$. Summing up, if $\beta_{1,t} = \beta_1\lambda^{t-1}$, then, from (\ref{eqsecond2}), (\ref{beta_1lambda^{t-1}}), and (\ref{eqmain2}), we obtain
\begin{eqnarray*}
R(T) & \le&   \frac{dD_{\infty}^2G_{\infty}}{2\alpha(1-\beta_1)}\left( \sum_{t=1}^{t_0} \sqrt{t} + \sqrt T\right)+  \frac{d D_{\infty}^2G_{\infty}}{2\alpha(1-\beta_{1})(1-\lambda)^2} + \frac{\alpha\sqrt{ \ln T +1}}{(1-\beta_1)^2\sqrt{1-\beta_2}(1-\gamma)} \sum_{i=1}^{d}\lVert{g_{1:T, i}}\rVert_2 .
\end{eqnarray*}
If $\beta_{1,t} = \frac{\beta_1}{t}$, then, from from (\ref{eqsecond2}), (\ref{frac{beta_1}{t}}), and (\ref{eqmain2}), we obtain
\begin{eqnarray*}
R(T) & \le&   \frac{dD_{\infty}^2G_{\infty}}{2\alpha(1-\beta_1)}\left( \sum_{t=1}^{t_0} \sqrt{t} + \sqrt T\right)+  \frac{d D_{\infty}^2G_{\infty}\sqrt T}{\alpha(1-\beta_{1})} + \frac{\alpha\sqrt{ \ln T +1}}{(1-\beta_1)^2\sqrt{1-\beta_2}(1-\gamma)} \sum_{i=1}^{d}\lVert{g_{1:T, i}}\rVert_2 ,
\end{eqnarray*}
which ends the proof.
\end{proof}

The following corollary shows that, when either $\beta_{1,t} = \beta_1\lambda^{t-1}$ or $\beta_{1,t}= 1/t$, $(1\le t \le T)$, where $0\le\beta_1< 1$ and $0<\lambda < 1$, the average regret of AMSGrad converges.
\begin{cor}\label{cor} With the same assumption as in Theorem \ref{mainthm_change_beta1},  AMSGrad achieves the following guarantee:
$$\lim_{T\to \infty} \frac{R(T)}{T} =  0.$$
\end{cor}
\begin{proof} The result is obtained by using Theorem \ref{new_con_AMSGrad} and the following fact: 
\begin{eqnarray*} 
\sum_{i=1}^{d}\lVert{g_{1:T, i}}\rVert_2 & = & \sum_{i=1}^{d}\sqrt{g_{1,i}^2 + g_{2,i}^2, ... + g_{T,i}^2}\\
& \le & \sum_{i=1}^{d}\sqrt{TG_{\infty}^2}\\
& = & dG_{\infty}\sqrt{T},
\end{eqnarray*}
where the inequality is from the assumption that $\lVert{g_t}\rVert_{\infty} \le G_{\infty}$ for all $t\in [T]$.
\end{proof}


\section{New version of AMSGrad optimizer: AdamX}\label{new_version}
Let $f_1, f_2,..., f_T: \mathcal F \to \mathbb R $ be an arbitrary sequence of convex cost functions. If the system $\{\beta_{1,t}\}_{1\le t\le T}\}$ is kept arbitrary, as in the setting of Theorem \ref{mainthmAMSGrad}, to ensure that the regret $R(T)$ satisfies $R(T)/T\to 0$, we suggest a new algorithm as follows.
 \begin{algorithm}[H]\label{alg3}
    \caption{AdamX: a new variant of Adam and AMSGrad.}\label{AMSGradnew}
    \begin{algorithmic}
    	\State\hspace{-\algorithmicindent} \textbf{Input:} $x_1\in \mathbb R^d$, step size $\{\alpha_t\}_{t=1}^T, \{\beta_{1,t}\}_{t=1}^T, \beta_2$
	\State \hspace{-\algorithmicindent} Set $m_0 = 0, v_0 = 0$, and $\hat v_0 = 0$
   	\For {$(t=1; t\le T; t\gets t+1)$}
		\State $g_{t} = \nabla f_{t}(x_{t})$ 
		\State $m_{t} = \beta_{1,t}\cdot m_{t-1} + (1-\beta_{1,t})\cdot g_{t}$ 
		\State $v_t = \beta_2\cdot v_{t-1} + (1-\beta_2)\cdot g^2_t$ 
		\State $\hat v_{1} = v_{1}$, $\hat v_{t} = \max\{\frac{(1-\beta_{1,t})^2}{(1-\beta_{1,t-1})^2}\hat v_{t-1} , v_{t}\}$ if $t\ge 2$, and $\hat V_t = \text{diag}(\hat v_t)$
		\State $x_{t+1} =  \prod_{\mathcal F, \sqrt{\hat V_t}}(x_t - \alpha_t \cdot m_t/\sqrt{\hat v_t}) $ 
    	\EndFor
    \end{algorithmic}
    \end{algorithm}
    
With this Algorithm \ref{AMSGradnew}, the regret is bounded as follows.
\begin{thm}\label{mainthm2} Let $x_t$ and $v_t$ be the sequences obtained from Algorithm \ref{AMSGradnew}, $\alpha_t = \frac{\alpha}{\sqrt{t}}$, $\beta_1 = \beta_{1,1}$, $\beta_{1,t} \le \beta_1$ for all $t\in[T]$ and $\frac{\beta_1}{\sqrt{\beta_2}} \le 1$. Assume that $\mathcal F$ has bounded diameter $D_{\infty}$ and $\lVert{\nabla f_t(x)}\rVert_{\infty} \le G_{\infty}$ for all $t\in [T]$ and $x\in \mathcal F$. For $x_t$ generated using the AdamX (Algorithm \ref{AMSGradnew}), we have the following bound on the regret:
\begin{eqnarray*}
R(T) & \le&  \frac{dD_{\infty}^2G_{\infty}}{2\alpha(1-\beta_{1})}\sqrt{T}  + \frac{dD_{\infty}^2G_{\infty}}{2\alpha(1-\beta_{1})} \sum_{t=2}^{T}\beta_{1,t}\sqrt{(t-1)} +  \frac{\alpha\sqrt{ \ln T +1}}{(1-\beta_1)^2\sqrt{1-\beta_2}(1-\gamma)} \sum_{i=1}^{d}\lVert{g_{1:T, i}}\rVert_2 ~.
\end{eqnarray*}
\end{thm}


To prove Theorem \ref{mainthm2}, we need the following Lemmas \ref{vtnew}, \ref{vt2}, and \ref{mainlem2}.
\begin{lem}\label{vtnew} For all $t\ge 1$, we have 
\begin{eqnarray}\label{vthat}  
\hat{v}_{t}& =& \max\left\{\frac{(1-\beta_{1,t})^2}{(1-\beta_{1,s})^2}v_s, \text {~for ~ all~}1\le s\le t\right\},
\end{eqnarray}
where $\hat{v}_{t}$ is in Algorithm \ref{AMSGradnew}.
\end{lem}
\begin{proof} We will prove (\ref{vthat}) by induction on $t$. Recall that by the update rule on $\hat{v}_{t}$, we have  $\hat v_1 \overset{\Delta}{=} v_1$ and $\hat v_t \overset{\Delta}{=} \max\{\frac{(1-\beta_{1,t})^2}{(1-\beta_{1,t-1})^2}\hat v_{t-1} , v_{t}\}$ if $t\ge 2$. Therefore,
\begin{eqnarray*} 
\hat v_2 & \overset{\Delta}{=} & \max\{\frac{(1-\beta_{1,2})^2}{(1-\beta_{1,1})^2}\hat v_{1} , v_{2}\}\\
&= & \max\{\frac{(1-\beta_{1,2})^2}{(1-\beta_{1,1})^2} v_{1} , v_{2}\}\\
&= & \max\{\frac{(1-\beta_{1,t})^2}{(1-\beta_{1,s})^2}v_s, 1\le s\le 2\}.
\end{eqnarray*} 
Assume that $$\hat{v}_{t-1} = \max\{\frac{(1-\beta_{1,t-1})^2}{(1-\beta_{1,s})^2}v_s, \text {~for ~ all~} 1\le s\le t-1\}$$
and the (\ref{vthat}) holds for all $1\le j\le t-1$. 
Since 
$$\hat{v}_{t}\overset{\Delta}{=} \max\{\frac{(1-\beta_{1,t})^2}{(1-\beta_{1,t-1})^2}\hat{v}_{t-1}, v_t\},$$ 
we have 
\begin{eqnarray*} 
\hat{v}_{t}& {=} & \max\{\frac{(1-\beta_{1,t})^2}{(1-\beta_{1,t-1})^2}\left(\max\{\frac{(1-\beta_{1,t-1})^2}{(1-\beta_{1,s})^2}\hat{v}_s, \text {~for ~ all~}1\le s\le t-1\}\right), v_t\}\\
& = & \max\{  \max \{\frac{(1-\beta_{1,t})^2}{(1-\beta_{1,t-1})^2}\frac{(1-\beta_{1,t-1})^2}{(1-\beta_{1,s})^2}{v}_s,\text {~for ~ all~}1\le s\le t-1\}, \frac{(1-\beta_{1,t})^2}{(1-\beta_{1,t-1})^2}v_t\}\\
& = & \max\{ \{\frac{(1-\beta_{1,t})^2}{(1-\beta_{1,s})^2} {v}_s, \text {~for ~ all~}1\le s\le t-1\}, \frac{(1-\beta_{1,t})^2}{(1-\beta_{1,t-1})^2}v_t\}\\
& = & \max\{ \frac{(1-\beta_{1,t})^2}{(1-\beta_{1,s})^2} {v}_s, \text {~for ~ all~}1\le s\le t\},
\end{eqnarray*}
which ends the proof.
\end{proof}

\begin{lem}\label{vt2} For all $t\ge 1$, we have $\sqrt{\hat{v}_{t}} \le \frac{G_{\infty}}{1-\beta_{1}}$, where $\hat{v}_{t}$ is in Algorithm \ref{AMSGradnew}.
\end{lem}
\begin{proof} By Lemma \ref{vtnew},
$$\hat{v}_{t} = \max\{\frac{(1-\beta_{1,t})^2}{(1-\beta_{1,s})^2}v_s, 1\le s\le t\}.$$ 
Therefore, there is some $1\le s\le t$ such that $\hat{v}_{t} = \frac{(1-\beta_{1,t})^2}{(1-\beta_{1,s})^2}v_s$. Hence, 
\begin{eqnarray*} 
\sqrt{\hat{v}_{t}} & = & \sqrt{\frac{(1-\beta_{1,t})^2}{(1-\beta_{1,s})^2}v_s}\\
& = & \sqrt{1-\beta_2}\left(\frac{1-\beta_{1,t}}{1-\beta_{1,s}}\right)\sqrt{\sum_{k=1}^{s}\beta_2^{s-k}g^2_{k}}\\
&\le&  \sqrt{1-\beta_2}\left(\frac{1-\beta_{1,t}}{1-\beta_{1,s}}\right)\sqrt{\sum_{k=1}^{s}\beta_2^{s-k} (\max_{1\le j\le s}{|g_{j}|})^2}\\
&=& G_{\infty}\sqrt{1-\beta_2}\left(\frac{1-\beta_{1,t}}{1-\beta_{1,s}}\right)\sqrt{\sum_{k=1}^{s}\beta_2^{s-k}}\\
& \le & G_{\infty}\sqrt{1-\beta_2}\left(\frac{1-\beta_{1,t}}{1-\beta_{1,s}}\right)\frac{1}{\sqrt{1-\beta_2}}\\
& = & \left(\frac{1-\beta_{1,t}}{1-\beta_{1,s}}\right)G_{\infty}\\
& \le & \frac{G_{\infty}}{1-\beta_{1}},
\end{eqnarray*}
which ends the proof.
\end{proof}

\begin{lem}\label{mainlem2} For the parameter settings and conditions assumed in Theorem \ref{mainthm2}, we have
$$\sum_{t=1}^T\frac{{m}^2_{t,i}}{\sqrt{t\hat{v}_{t,i}}} \le \frac{\sqrt{ \ln T +1} }{(1-\beta_1)\sqrt{1-\beta_2}(1-\gamma)}\lVert{g_{1:T, i}}\rVert_2 .$$
\end{lem}

\begin{proof}
Since for all $ t\ge 1$
$$\hat{v}_{t,i} = \max\{\frac{(1-\beta_{1,t})^2}{(1-\beta_{1,s})^2}v_s (1\le s\le t)\},$$ by Lemma \ref{vtnew},
we have $\hat{v}_{t,i} \ge v_{t,i}$, and hence the proof is the same as that of Lemma \ref{mainlem}.
\end{proof}


\begin{proof}[Proof of Theorem \ref{mainthm2}] Similarly to the proof of Theorem \ref{mainthm_change_beta1}, we need to bound (\ref{eqmain}), (\ref{eqsecond}), and (\ref{eqthird}). By using Lemma \ref{mainlem2}, we obtain the same bound for (\ref{eqsecond}) as in the proof of Theorem \ref{mainthm_change_beta1}, that is,

\begin{eqnarray*}
(\ref{eqsecond}) = \sum_{i=1}^{d} \sum_{t=1}^{T} \frac{\alpha_t}{1-\beta_{1}} \frac{ m_{t,i}^2}{\sqrt{\hat v_{t,i}}} 
& = & \frac{\alpha}{1-\beta_{1}}\sum_{i=1}^{d} \sum_{t=1}^{T} \frac{ m_{t,i}^2}{\sqrt{t\hat v_{t,i}}}\\
& \le &\frac{\alpha\sqrt{ \ln T +1}}{(1-\beta_1)^2\sqrt{1-\beta_2}(1-\gamma)} \sum_{i=1}^{d}\lVert{g_{1:T, i}}\rVert_2,
\end{eqnarray*}
where the last inequality is by Lemma \ref{mainlem2}. 
Now we bound (\ref{eqthird}). By the assumption that $\lVert{x_m-x_n}\rVert_{\infty} \le D_{\infty}$ for any $m,n\in \{1,...,T\}$,  $\alpha_t = \alpha/\sqrt{t}$, and $\beta_{1,t} = \beta_1\lambda^{t-1} \le \beta_1 \le 1$, we obtain
\begin{eqnarray*}
(\ref{eqthird}) = \sum_{i=1}^{d} \sum_{t=2}^{T}\frac{\beta_{1,t}\sqrt{\hat{v}_{t-1,i}}}{2\alpha_{t-1}(1-\beta_{1,t})}  (x_{t,i} - x^{*}_{,i})^2
& \le & \frac{D_{\infty}^2}{2\alpha(1-\beta_{1})} \sum_{i=1}^{d} \sum_{t=2}^{T}\beta_{1,t}\sqrt{(t-1)\hat{v}_{t-1,i}}.
\end{eqnarray*}
Therefore, from Lemma \ref{vt2}, we obtain
\begin{eqnarray*}
(\ref{eqthird}) &\le & \frac{dD_{\infty}^2G_{\infty}}{2\alpha(1-\beta_{1})^2}  \sum_{t=2}^{T}\beta_{1,t}\sqrt{(t-1)}.
\end{eqnarray*}
Finally, we will bound (\ref{eqmain}). By the inequality (\ref{eqtemp2}) and replacing $\alpha_t = \frac{\alpha}{\sqrt t} (1\le t\le T)$, we obtain 
\begin{eqnarray*}
 (\ref {eqmain}) 
 & \le &  \sum_{i=1}^{d} \frac{ \sqrt{\hat{v}_{1,i}}}{2\alpha(1-\beta_{1})}   (x_{1, i} - x^{*}_{,i})^2  + 
 \frac{1}{2\alpha}\sum_{i=1}^{d} \sum_{t=2}^{T} (x_{t, i} - x^{*}_{,i})^2 \left( \frac{\sqrt{t \hat{v}_{t,i}}}{1-\beta_{1,t}} - \frac{\sqrt{(t-1)\hat{v}_{t-1,i}}}{1-\beta_{1,t-1}}  \right)
\end{eqnarray*}
Moreover, by the update rule of Algorithm \ref{AMSGradnew}, we have
 $$\hat v_{t,i} = \max\{\frac{(1-\beta_{1,t})^2}{(1-\beta_{1,t-1})^2}\hat v_{t-1,i}, v_{t,i}\}.$$ 
 Therefore,
$\hat v_{t,i} \ge \frac{(1-\beta_{1,t})^2}{(1-\beta_{1,t-1})^2}\hat v_{t-1,i}$, and hence 
\begin{eqnarray*}
\frac{\sqrt{t \hat{v}_{t,i}}}{1-\beta_{1,t}} - \frac{\sqrt{(t-1)\hat{v}_{t-1,i}}}{1-\beta_{1,t-1}} 
& \ge& \frac{\sqrt{t\frac{(1-\beta_{1,t})^2}{(1-\beta_{1,t-1})^2}\hat v_{t-1,i}}}{1-\beta_{1,t}} - \frac{\sqrt{(t-1)\hat{v}_{t-1,i}}}{1-\beta_{1,t-1}}\\
& = & \frac{\sqrt{t \hat{v}_{t-1,i}}}{1-\beta_{1,t-1}} - \frac{\sqrt{(t-1)\hat{v}_{t-1,i}}}{1-\beta_{1,t-1}} \\
& > & 0.
\end{eqnarray*}
Now by the positivity of the essential formula $\frac{\sqrt{t \hat{v}_{t,i}}}{1-\beta_{1,t}} - \frac{\sqrt{(t-1)\hat{v}_{t-1,i}}}{1-\beta_{1,t-1}}$, we obtain
\begin{eqnarray*}
 (\ref {eqmain}) 
 & \le &  
 \frac{D_{\infty}^2 }{2\alpha}\sum_{i=1}^{d} \frac{ \sqrt{\hat{v}_{1,i}}}{1-\beta_{1}}  +
 \frac{D_{\infty}^2 }{2\alpha}\sum_{i=1}^{d} \sum_{t=2}^{T} \left( \frac{\sqrt{t \hat{v}_{t,i}}}{1-\beta_{1,t}} - \frac{\sqrt{(t-1)\hat{v}_{t-1,i}}}{1-\beta_{1,t-1}}  \right)\\
 & = & \frac{D_{\infty}^2}{2\alpha}\sum_{i=1}^{d} \frac{\sqrt{T\hat{v}_{T,i}} }{1-\beta_{1,T}}\\
 & \le &  \frac{dD_{\infty}^2G_{\infty}}{2\alpha(1-\beta_{1})^2}\sqrt{T}~,
 \end{eqnarray*}
where the last inequality is by Lemma \ref{vt2}. Hence we obtain the desired upper bound for $R(T)$.
\end{proof}

\begin{cor}\label{ge_cor}
With the same assumption as in Theorem \ref{mainthm2}, and for all $0 \le \beta_{1,t} < 1$ satisfying 
$$\lim_{T\to \infty}\frac{\sum_{t=2}^{T}\beta_{1,t}\sqrt{t-1}}{T} = 0,$$ AdamX achieves the following guarantee:
$$\lim_{T\to \infty} \frac{R(T)}{T}=0.$$
\end{cor}
\begin{proof}
By Theorem \ref{mainthm2}, it is sufficient to consider the term 
$$\frac{dD_{\infty}^2 G_{\infty}}{2\alpha(1-\beta_{1})^2} \sum_{t=2}^{T}\beta_{1,t}\sqrt{t-1}$$
on the right hand side of the upper bound for $R(T)$ in Theorem \ref{mainthm2}. Because $\frac{dD_{\infty}^2 G_{\infty}}{2\alpha(1-\beta_{1})^2}$ is bounded and does not depend on $T$, the statement follows.
\end{proof}
When either $\beta_{1,t} = \beta_1\lambda^{t-1}$ for some $\lambda \in (0,1)$, or $\beta_{1,t} = \frac{1}{t}$ in Theorem \ref{mainthm2}, we obtain the following  guarantee that the average regret of AdamX converges.

\begin{cor}\label{bound} With the same assumption as in Theorem \ref{mainthm2}, and either $\beta_{1,t} = \beta_1\lambda^{t-1}$ for some $\lambda \in (0,1)$, or  $\beta_{1,t} = \frac{1}{t}$, 
AdamX achieves the following guarantee:
$$\lim_{T\to \infty} \frac{R(T)}{T} =  0.$$ 
\end{cor}

\begin{proof} By Corollary  \ref{ge_cor}, it is sufficient to consider the term 
$$\sum_{t=2}^{T}\beta_{1,t}\sqrt{t-1}.$$
When $\beta_{1,t} = \beta_1\lambda^{t-1}$ for some $\lambda \in (0,1)$, we have
\begin{eqnarray}\label{last1}
 \sum_{t=2}^{T}\beta_{1,t}\sqrt{t-1}
& =&    \sum_{t=2}^{T}\beta_1\lambda^{t-1}\sqrt{t-1}\\
 &\le&    \sum_{t=2}^{T}\sqrt{(t-1)} \lambda^{t-1} \nonumber\\
& \le &   \sum_{t=2}^{T}t \lambda^{t-1}\nonumber\\
 &\le & \frac{1}{(1-\lambda)^2} \nonumber
\end{eqnarray}
where the first inequality is from the property that $\beta_1 \le1$, and
the last inequality is from Lemma \ref{Taylor}. 
When $\beta_{1,t} = \frac{1}{t}$, we obtain 
\begin{eqnarray}\label{last2}
 \sum_{t=2}^{T}\beta_{1,t}\sqrt{t-1}
&= &  \sum_{t=2}^{T}\frac{\sqrt{t-1}}{t}\\
&\le &  \sum_{t=2}^{T}\frac{1}{\sqrt{t}}\nonumber\\
 &\le &  2\sqrt T,
 \nonumber
\end{eqnarray}
where the last inequality is from Lemma \ref{sqrt}. Now, by combining (\ref{last1}) and (\ref{last2}) with Corollary \ref{ge_cor}, we obtain the desired result.
\end{proof}

\section{{Experiments}}
While we consider our main contributions as the theoretical analyses on AMSGrad and AdamX in the previous sections, we provide experimental results in this section for AMSGrad and AdamX. Concretely, we use the PyTorch code for AMSGrad\footnote{https://pytorch.org/docs/stable/\_modules/torch/optim/adam.html} via setting the boolean flag {\tt amsgrad = True}. The code for AdamX is based on that of AMSGrad, with corresponding modifications as in Algorithm \ref{AMSGradnew}. The parameters for AMSGrad and AdamX are identical in our experiments, namely $(\beta_1, \beta_2)=(0.9, 0.999)$, the term added to the denominator to improve numerical stability is $\epsilon = 10^{-8}$, and and additionally we set $\beta_{1,t} = \beta_1\lambda^{t-1}$ with $\lambda=0.001$ to make use of Corollary \ref{bound} on the convergence of AdamX. 

The learning rate is scheduled for both optimizers AMSGrad and AdamX as follows: $10^{-3}$, $10^{-4}$, $10^{-5}$, $10^{-6}$, $10^{-6}/2$ if the epoch is correspondingly in the ranges $[0,80]$, $[81, 120]$, $[121, 160]$, $[161, 180]$, $[181, 200]$. We use CIFAR\footnote{https://www.cs.toronto.edu/~kriz/cifar.html}-10 (containing  50000 training images and 10000 test images of size $32\times 32$) as the dataset and the residual networks ResNet18 \cite{Resnet_cvpr2016} and PreActResNet18 \cite{Resnet_ECCV2016} for training with batch size is 128. The testing result is given in Figure \ref{experiment_fig} where one can see that AMSGrad and AdamX behaves similarly, which supports our theoretical results on the convergence of both AMSGrad (Section \ref{new_con_AMSGrad}) and AdamX (Section \ref{new_version}).

\begin{figure}[t]
\includegraphics[scale=0.6]{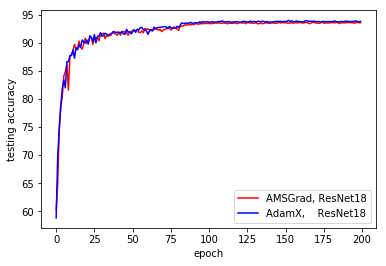} 
\includegraphics[scale=0.6]{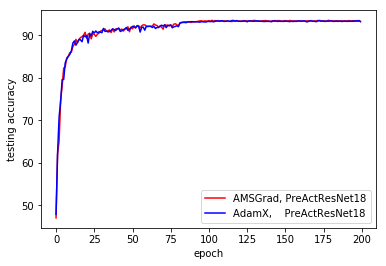} 
\caption{Testing accuracies over CIFAR-10 using AMSGrad and AdamX, with different neural network models.}\label{experiment_fig}
\end{figure}

{\color{black}
\section{Conclusion}
We have shown that  the convergence proof of AMSGrad \cite{AMSGrad} is  problematic, and  presented various fixes for it, which include a new and slightly modified version called AdamX. Along the lines, we also observe  that the issue has been neglected in various works such as in \cite[Theorem 10.5]{KingmaB14}, \cite[Theorem 4]{LuoXiongLiu}, \cite[Theorem 4.4]{BoGoWe}, \cite[Theorem 4.2]{Padam}. Our work helps ensure  the theoretical foundation of those optimizers.
}


\end{document}